\documentclass[10pt]{article}

\usepackage{times}
\usepackage{balance}
\usepackage{graphicx}

\usepackage[sumlimits]{amsmath}
\usepackage{amssymb}
\usepackage{amsthm}

\newtheorem{theorem}{Theorem}
\newtheorem{lemma}{Lemma}

\newcommand{\tomwidth}{0.08}
\newcommand{\tomheight}{0.09}

\title{Secure SURF with Fully Homomorphic Encryption}

\begin{document}

\author{Thomas Shortell \and Ali Shokoufandeh}

\date{Drexel University\\
  3141 Chestnut St.\\
  Philadelphia, PA 19014\\
  tms38@drexel.edu,ashokouf@cs.drexel.edu
  }

\maketitle

\begin{abstract}
  Cloud computing is an important part of today's world because offloading computations is a method to reduce costs. 
  In this paper, we investigate computing the Speeded Up Robust Features (SURF) using Fully Homomorphic Encryption (FHE).
  Performing SURF in FHE enables a method to offload the computations while maintaining security and privacy of the original data. 
  In support of this research, we developed a framework to compute SURF via a rational number based format compatible with FHE.
  Although floating point ($\mathbb{R}$) to rational numbers ($\mathbb{Q}$) conversion introduces error, our research provides tight bounds on the magnitude of error in terms of parameters of FHE.
  We empirically verified the proposed method against a set of images at different sizes and showed that our framework accurately computes most of the SURF keypoints in FHE. 
\end{abstract}

\section{Introduction}
\label{sec:introduction}

Cloud computers are a common modern day paradigm of offloading computations from a local desktop to a server farm.
Many of these cloud computing resources are honest participants in terms of privacy and security.
However, adversities can compromise an honest cloud computing resource and the user loses the privacy of their data.
This risk can be mitigated by encrypting the data so that a compromise does not immediately provide the private data to an adversity.
However, processing encrypted data is not straightforward.
We defeat this threat by using Fully Homomorphic Encryption (FHE) as a method to process encrypted data and make use of a cloud computing resource with some level of privacy.

Our focus on secure signal processing leads us to Speeded Up Robust Features (SURF) \cite{bay2008speeded} \cite{bay2006surf}.
Our FHE-based solution takes in an image, encrypts it, generates the scale space pyramid of SURF in an encrypted image, decrypts the pyramid, and then extracts keypoints.
Our key contribution is the computation of the scale space pyramid in the encrypted domain. 
We do not perform the keypoint extraction in the encrypted domain because FHE schemes do not currently provide a method to compare values in the encrypted domain.
However, it is potentially an open problem to process keypoints in the encrypted domain (particularly the LU decomposition for aligning the keypoint).
Another contribution of this paper is to use the integer space of FHE in a way to process the feature extraction algorithm that uses real values.
Our implementation uses encrypted rational values in a clever way to allow the scale space pyramid to be computed in the encrypted domain. 
We particularly focus on errors that can be introduced by use of the rational format compared to a floating point representation.
The rational representation was chosen based on the ease of implementation in FHE. 

Fully Homomorphic Encryption (FHE) is a concept where processing of data is conducted while the data is encrypted.
FHE requires that the encrypted processing can continue indefinitely.
Additionally, FHE requires that the homomorphic operations occur over addition and multiplication.
In the Related Works section, we identify some additive-only encryption schemes. 
The first scheme capability of FHE was developed by Gentry \cite{gentry:computing}.
Additional schemes have been developed over the years.
We use a scheme developed in 2013 by Gentry, Sahai, and Waters \cite{gentry2013homomorphic} (called GSW2013).
Security of the scheme is based on the learning with errors problem \cite{regev2009lattices}.
The GSW13 scheme provides important capabilities that we are going to use to implement secure SURF.
The scheme encrypts integers over a ring defined by a modulus $q$.
Once encrypted, the integers can be added and multiplied together.
Subtraction is available with two's compliment integers.
There is also the capability to multiply an encrypted ciphertext by a constant.
These capabilities provide the building blocks to perform secure SURF.

\section{Related Work}
\label{sec:related}

Over the past decade growing interest in privacy has increased interest in secure signal processing. 
Troncoso-Pastoriza and Perez-Gonzalez have an excellent survey of the problems and solution techniques for privacy in cloud computing \cite{troncoso2013secure}.
There are many other examples among varying fields; secure signal processing \cite{knevzevic2013signal}, biometrics \cite{wang2012theoretical}, and \cite{puech2012emerging}. 
Many specific examples exist that use  the Paillier encryption scheme (see \cite{hsu2011homomorphic} for scheme definition).
This scheme is not an FHE scheme as it provides only additive homomorphic operations (i.e. no ciphertext-ciphertext multiplication).
FHE schemes provide both additive and multiplicative operations over ciphertexts. 
Hsu, Lu, and Pei used Paillier to implement a privacy preserving SIFT \cite{hsu2011homomorphic}.
An encrypted version of SURF also exists using the Paillier scheme \cite{bai2014surf}.
The main difference here is that we use FHE instead of Paillier.
A few other examples of using Paillier exist in signal processing  \cite{lathey2013homomorphic} \cite{mohanty2013scale}.
Other encryption techniques also exist; particularly with two party computations while hiding data from both sides, ex. \cite{shashanka2006secure}.

FHE schemes have been evolving since the first scheme was developed by Gentry in 2009 \cite{gentry2009fully}.
Over the next few years, the schemes have improved both the time and space complexity and went from binary number implementation to an integer ring \cite{brakerski2012fully} \cite{brakerski2012leveled} \cite{brakerski2011efficient} \cite{gentry2013homomorphic}.
Our choice of the 2013 scheme is due to its ability to process ciphertexts within an integer ring. 
Other work has been done on using FHE in signal processing:
Shortell and Shokoufandeh used FHE to implement a brightness/contrast filter \cite{shortell215secruresignal}.
This implementation uses scaling factor representation vice a rational representation.
We will show that our rational representation is more practical in comparison.

\section{Construction of SURF in FHE}
\label{sec:impl}

\subsection{Numerical Format}
\label{sec:impl:format}

As we mentioned, the FHE scheme only provides encryption, decryption, and processing of integers within the ring $\mathbb{Z}_q$.
Any mathematical function that requires real value computations would not be natively possible in the FHE scheme.
This requires the need to represent real value in the integer ring.
Multiple solutions exists such as a fixed point factor with each solution having its own advantages and disadvantages.
We chose to use a rational representation of an integer pair because we can support many basic operations (+,-,$\times$,$/$) of fractional numbers.

Rational representation requires a tuple of two integers inside $\mathbb{Z}_q$.
With rational numbers (in $\mathbb{Q}_q$), we can represent a real value up to some approximation factor.
This will obviously introduce error into the calculations compared to floating point.
The formal representation of the format is
\begin{equation}
  \{ u\cdot v^{-1} \mid u, v \in \mathbb{Z}_q, v \neq 0 \}
\end{equation}
Using this model, rational addition, subtraction, multiplication, and division are all well defined in terms of addition, subtraction, and multiplication of the individual integers in the tuple.
Since the FHE scheme supports these three operations, we can use a rational format for secure SURF. 
This format has an advantage over a fixed point binary format in that there are less computations to perform for computing the end result. 
There is a major disadvantage of this format: ciphertexts do not have any idea what the denominator ($v$) is and cannot be reduced with the numerator.
So values greater than one will always increase the denominator which will at one point overload the $\mathbb{Z}_q$ space and cause an incorrect value to be computed.
We work around this potential issue by making smart decisions to implement SURF in the encrypted domain.

\subsection{SURF Implementation}
\label{sec:impl:surf}

We now focus on the detailed SURF implementation in FHE using the rational format.
For SURF, we need to carefully consider the expansion of the denominator in all calculations.
Our first smart decision is to choose the initial image format to be integer only; thus we start the denominator as one and remains at one until the entire integral image is generated. 
This is very important as there are $O(mn)$ operations for the integral image (assuming an $m\times n$ image). 
The integral image is computed by starting at the top, left point and moving towards the bottom, right point.
We can generate in place by copying the image and adding an initial encrypted zero row and column.
Using a simple equation, we are able to compute the entire integral image in $O(n^2)$.
\begin{equation}
  \mbox{int}_{i,j} = \mbox{int}_{i,j} + \mbox{int}_{(i-1),j} + \mbox{int}_{i,(j-1)} - \mbox{int}_{(i-1),(j-1)}
\end{equation}

We move on to computing the scale space pyramid of determinants and traces.
Just as in \cite{bay2008speeded}, we compute $2 \times 2$ Hessian matrix with assistance of Haar patterns.
Omitting the details, we highlight here that we follow the same path of computing regions from integral image for the Haar patterns to the Hessian matrix. 
We start with computing regions first.
The region summation benefits from the integral image having denominator of one.
This means the three or four computations will not increase the denominator until the multiplier constants are used.
Given four points of a rectangular region ($A$, $B$, $C$, and $D$), we compute the sum as:
\begin{equation}
  \mbox{rs}_k = \mbox{int}_{A} - \mbox{int}_{B} - \mbox{int}_{C} + \mbox{int}_{D} ,
\end{equation}
where we have used $k$ to indicate a specific region. 
As we mentioned, there will be three ($D_{xx}$,$D_{yy}$) or four ($D_{xy}$) of these summations to be combined. 
The basic equation (for $D_{xx}$ or $D_{yy}$) is
\begin{equation}
  c_1 \cdot \mbox{rs}_1 + c_2 \cdot \mbox{rs}_2 + c_3 \cdot \mbox{rs}_3 ,
\end{equation}
where the three constants are rational values depending on the current octave and layer of the scale space pyramid.
It is at this point in the computations that we introduce a denominator greater than one.
We control the denominator of our rational numbers here by using a consistent denominator for the three constants and separating it out from the additions:
\begin{equation}
  \frac{1}{v} \cdot \left( c_{1,u} \cdot \mbox{rs}_1 + c_{2,u} \cdot \mbox{rs}_2 + c_{3,u} \cdot \mbox{rs}_3 \right).
\end{equation}
$D_{xy}$ uses a similar equation with an additional term.
We will see the importance of having  a consistent $v$ for all computations here for the determinant. 
Determinant is computed as follows:
\begin{equation}
  D_{xx} \cdot D_{yy} - 0.81 \cdot D_{xy}^2 .
\end{equation}
Remembering we have a common denominator from the Haar patterns, we can extract out the denominator to minimize growth. 
Using $v$ from the previous paragraph:
\begin{equation}
  \label{eq:enc:det}
  \frac{1}{v^2} \left( D_{xx,u} \cdot D_{yy,u} - \frac{81}{100} \cdot D_{xy,u} \cdot D_{xy,u} \right) .
\end{equation}
Computing the trace is a very simple calculation given the Haar patterns.
The equation is
\begin{equation}
  \label{eq:enc:trace}
  D_{xx} + D_{yy} ,
\end{equation}
which does not require any denominator movement as this is a final equation that gets decrypted.

\paragraph{Re-encryption}

As a final note, we need to discuss the expansion of the image size. 
The FHE scheme only can support a certain number of homomorphic operations before re-encryption is needed to refresh the ciphertext from decryption noise.
Without this, the noise will cause decryption to fail \cite{gentry2013homomorphic}.
In the case of SURF, the longest sequence of computations occurs during the integral image.
Scale space pyramid's computations are more parallel in nature compared to the sequential nature of the integral image. 
When the image size becomes larger, we can strategically refresh the ciphertext in segments of the integral image. 

\subsection{Detailed Error Analysis}
\label{sec:impl:error}

Our rational format approach introduces errors into the computations.
While floating point numbers have limited accuracy, the inaccuracy can be so small that the user does not notice it.
We need to consider how our framework introduces errors and what needs to be done to contain them.
Therefore, we need to bound both the numerator and the denominator within the FHE scheme's limits. 
We develop a theorem to enable users to understand what modulus must be selected in the FHE scheme to enable correct computations.
The FHE scheme used in this paper introduces error or noise in keys that enables the encryption to be secure.
This error/noise is not the same as our introduced error and does not affect the rational numbers (until decryption fails) because the plaintext integers are extracted correctly.

\paragraph{Introduction of Error}

Error is introduced when we convert from floating point to rational.
It is important to remember we strategically kept our initial image in integer space, so we do not introduce additional error from computations. 
However, there are two more locations in the process that will introduce error: the fractional constants that are used when computing Haar patterns and constant $\frac{81}{100}$ that is used in the determinant calculation.
One concern with the error is balancing the numerator and denominator bounds to obtain accurate results. 

\begin{lemma}
  \label{lemma:error}
  Given the SURF implementation and an accuracy bounded by $\Delta$, the total error introduced will be
  \begin{equation*}
    \Delta \cdot B \cdot m \cdot n \cdot ( 3 D_{xx} + 3 D_{yy} - 0.81 \cdot 8 \cdot D_{xy} )
  \end{equation*}
  for the determinant and
  \begin{equation*}
    2 \cdot \Delta \cdot B \cdot m \cdot n
  \end{equation*}
  for the trace using an image size of $m$ by $n$ with integers in the integral image bounded by $B$.
\end{lemma}

\begin{proof}
  The error begins in the Haar pattern calculations (assuming accuracy is limited by $\Delta$):
  \begin{equation}
    (c_1 + \Delta) \cdot r_1 + (c_2 + \Delta) \cdot r_2 + (c_3 + \Delta) \cdot r_3 .
  \end{equation} 
  This means a single Haar pattern will introduce $\Delta \cdot (r_1+r_2+r_3)$ error. 
  Combining these into the determinant equation:
  \begin{eqnarray}
    = (D_{xx} + \Delta \cdot (r_1+r_2+r_3) ) \cdot (D_{yy} + \Delta \cdot (r_1+r_2+r_3) ) \nonumber\\
    - .81 \cdot \left( D_{xy} + \Delta \cdot (r_1+r_2+r_3+r_4) \right)^2\\
    = D_{xx} \cdot D_{xy} - .81 \cdot D_{xy}^2 + \Delta \cdot (r_1+r_2+r_3) \cdot (D_{xx}+D_{yy}) \nonumber\\
    - .81 \cdot (r_1+r_2+r_3+r_4) \cdot (D_{xy} + D_{xy})    .
  \end{eqnarray}
  We can convert the $r_x$s to worst case values of $B \cdot m \cdot n$ to be the max possible value they can be.
  Combining the error terms:
  \begin{equation}
    \Delta \cdot B \cdot m \cdot n \cdot ( 3 D_{xx} + 3 D_{yy} - 0.81 \cdot 8 \cdot D_{xy} ) .
  \end{equation}
  The trace equation is much easier.
  Starting with
  \begin{equation}
    D_{xx} + \Delta \cdot (r_1+r_2+r_3) + D_{yy} + \Delta \cdot (r_1+r_2+r_3) ,
  \end{equation}
  combining terms and using the max bound for the $r_x$s, yields:
  \begin{equation}
    D_{xx} + D_{yy} + 2 \cdot \Delta \cdot B \cdot m \cdot n .
  \end{equation}
\end{proof}

\paragraph{Denominator Bounding}

We need to bound the denominator.
As discussed in Section~\ref{sec:impl:surf}, we strategically selected the denominator in areas to minimize its ability to increase.
At the conclusion of the encrypted processing, we are calculating two values using equations~\ref{eq:enc:det} and~\ref{eq:enc:trace}.
Obviously, our bound needs to be the max of these two.
While it was easy to determine the values of the denominators, the real problem is the FHE scheme's max value.
We generate this Lemma to support the FHE scheme. 
\begin{lemma}
  \label{lemma:denominator}
  Given the SURF implementation and a base denominator of $\; V \;$, the max value of the denominator is $100 \cdot V^2$
\end{lemma}

\begin{proof}  
The integral image denominator is the same for the entire process which is maintained for the region summations.
The first strategic decision we made was to use a common denominator and move it out of the additions to minimize the denominator increase.
Therefore when a Haar pattern has been computed, there will be a denominator set to a value $V$.
From the trace equation (Eq.~\ref{eq:enc:trace}): knowledge of rational calculations and a common denominator of $V$ leads to a final denominator of $V^2$.
Working on the determinant equation, we strategically move the $V^2$ out of the computations (using the common denominator from Haar patterns).
The only remaining denominator is the $100$ from the constant.
This gives us a final denominator of $100 \cdot V^2$.
\end{proof}

\paragraph{Numerator Bounding}

Our next focus is to bound the numerator of the determinant and trace equations (\ref{eq:enc:det} and~\ref{eq:enc:trace}).

\begin{lemma}
  \label{lemma:numerator}
  Given the SURF implementation, the max value of the numerator is $1296 \cdot B^2 \cdot m^2 \cdot n^2$, where $B$ is the bound on the integer values of the image, $m \times n$ is the image size. 
\end{lemma}

\begin{proof}
  Again using Eqs.~\ref{eq:enc:det} and~\ref{eq:enc:trace}, we identify bounds separately. 
  This requires determining the max values of the integral image, region sums, and Haar patterns.
  We can independently work on bounding each individual Haar patterns and then going back to the region sums and integral image.
  The trace numerator is
  \begin{equation}
    D_{xx} \cdot V + D_{yy} \cdot V .
  \end{equation}
  The determinant numerator is:
  \begin{equation}
    \label{eq:det:numerator}
    D_{xx} \cdot D_{yy} \cdot 100 - 81 \cdot D_{xy} \cdot D_{xy} .
  \end{equation}
  In both cases, we need to determine the bounds for each of the Haar patterns.
  $D_{xx}$ and $D_{yy}$ will have the same bound and $D_{xy}$ will be very similar.
  The numerator of the $D_{xx}$ is
  \begin{equation}
    \label{eq:lemmanum:sum}
    c_1 \cdot r_1 + c_2 \cdot r_2 + c_3 \cdot r_3 .
  \end{equation}
  
  For the integral image, an image of size $m \times n$ drives $O(mn)$ calculations. 
  Using $B$ as the bound on the integer space, we have $B \cdot m \cdot n$ as the bound of the max value. 
  Each point ($r_x$) will be bounded above by the actual location; note that the constants $c_x$ will add up to zero but the summation in Eq.~\ref{eq:lemmanum:sum} will not be zero.
  Since we are looking for the worst case, the region sums will be just the size of the region. % ($i \cdot j$).
  To bound this from above, we will assume the constants are positive so that the region is a combined sum.
  This means a less tight bound overall. 
  We will designate the size of the region in the worst case to be less than $m\cdot n$; so that we can use this value as a bound. 
  This gives us a bound on the $D_{xx}$ ($D_{yy}$):
  \begin{equation}
    3 \cdot B \cdot m \cdot n .
  \end{equation}
  $D_{xy}$ is very similar:
  \begin{equation}
    4 \cdot B \cdot m \cdot n .
  \end{equation}

  We can combine these two results into Eq.~\ref{eq:det:numerator}:
  \begin{equation}
    \label{eq:lemmanum:num}
    \left( 3 \cdot B \cdot m \cdot n \right) \cdot \left( 3 \cdot B \cdot m \cdot n \right) \cdot 100 - 81 \cdot \left( 4 \cdot B \cdot m \cdot n \right)^2 ,
  \end{equation}
  where we keep the two sides separate as if either one goes over the limit the calculation will fail.
  Splitting Eq.~\ref{eq:lemmanum:num}, we obtain
  \begin{eqnarray}
    900 \cdot B^2 \cdot m^2 \cdot n^2  \\
    1296 \cdot B^2 \cdot m^2 \cdot n^2 .
  \end{eqnarray} 
  Obviously, the second equation will be the larger of the two and will be the one that absolutely needs to be satisfied.
  Thus we have a bound on the numerator. 
\end{proof}

\paragraph{Modulus Theorem}

The importance of the following theorem is two-fold: first is identifying error in the output of the scheme (accuracy) and second is bounding the framework to properly select values to ensure proper output.
A user can select a modulus $q$ and a rational fraction denominator $V$ based on the size of the input image while considering the need of SURF's fractional values. 

\begin{theorem}
  FHE implementation of SURF correctly calculates the scale-space pyramid with accuracy
  \begin{equation*}
    \Delta \cdot B \cdot m \cdot n \cdot ( 3 D_{xx} + 3 D_{yy} - 0.81 \cdot 8 \cdot D_{xy} )
  \end{equation*}
  given that the following holds true:
  \begin{eqnarray*}
    100 \cdot V^2 & < & \frac{q}{2} \\
    1296 \cdot B^2 \cdot m^2 \cdot n^2 & < & \frac{q}{2}
  \end{eqnarray*}
\end{theorem}

\begin{proof}
  Lemma~\ref{lemma:error} proves the accuracy equation. 
  Next, the numerator and denominator need to be bounded within the modulus ring of the FHE scheme.
  If the numerator or denominator exceeds the modulus ring, the framework will not work. 
  Remember that $q$ is the modulus of the ring and that negative numbers are enabled by two's complement, so the numerator and denominator must be less than $\frac{q}{2}$.
  In Lemmas~\ref{lemma:denominator} and~\ref{lemma:numerator} we determined the max value of the numerator and denominator of the final determinant value.
  Using the modulus bound we have:
  \begin{eqnarray}
    100 \cdot V^2 & < & \frac{q}{2} \\
    1296 \cdot B^2 \cdot m^2 \cdot n^2 & < & \frac{q}{2}
  \end{eqnarray}
  .
\end{proof}

\section{Results}
\label{sec:results}

We now wish to discuss the results of implementing the Secure SURF.
We have two focuses: first is on how well the keypoints can be extracted from the scale space pyramid.
Second is the time and space complexity of the implementation of FHE and thus SURF.

\paragraph{Keypoint Comparison}

Since we are obtaining SURF keypoints from a scale space pyramid computed in the encrypted domain, we need to compare them against an unencrypted version. 
For our test cases, we used a set of 11 images scaled down to $32 \times 32$ and $64 \times 64$ in size.
This arbitrary scaling was done to show the capability of the framework.
Our tests were setup with the FHE scheme having a ring modulus q of $256^7$ and we set $V = 10000$.

\begin{figure}
  \centering
  \includegraphics[width=\tomwidth\textwidth,height=\tomheight\textheight]{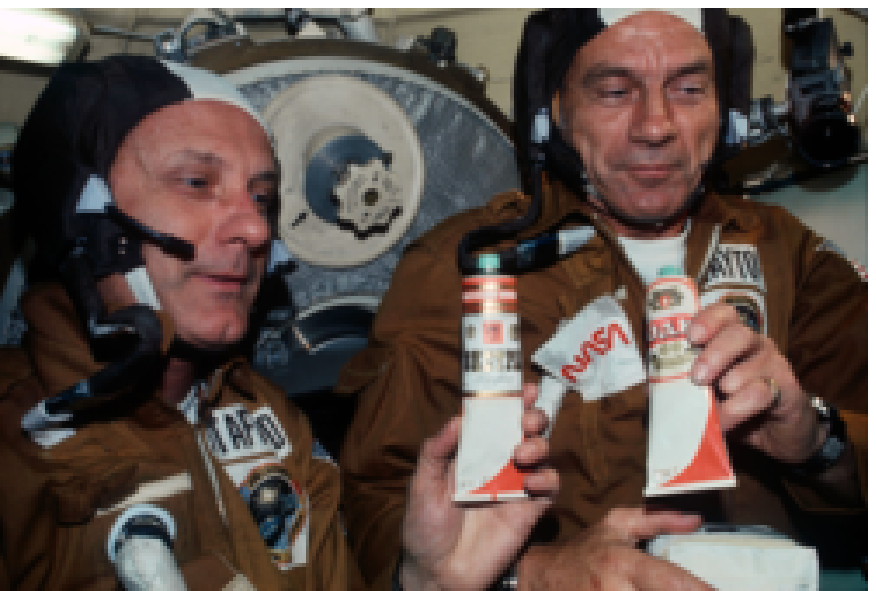}
  \hfill
  \includegraphics[width=\tomwidth\textwidth,height=\tomheight\textheight]{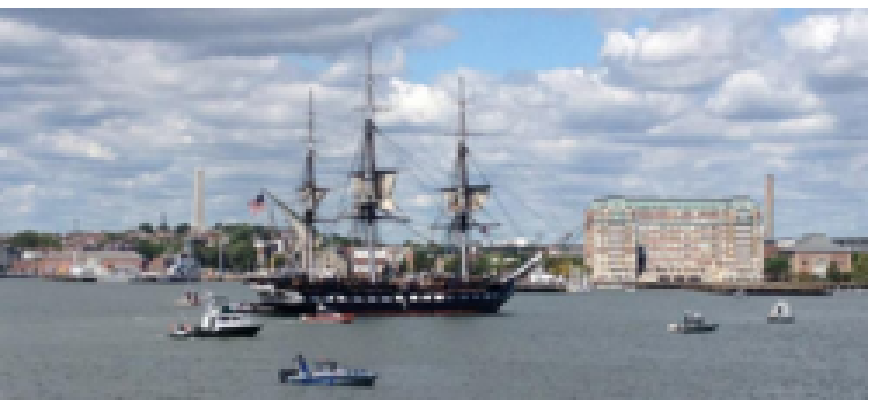}
  \hfill
  \includegraphics[width=\tomwidth\textwidth,height=\tomheight\textheight]{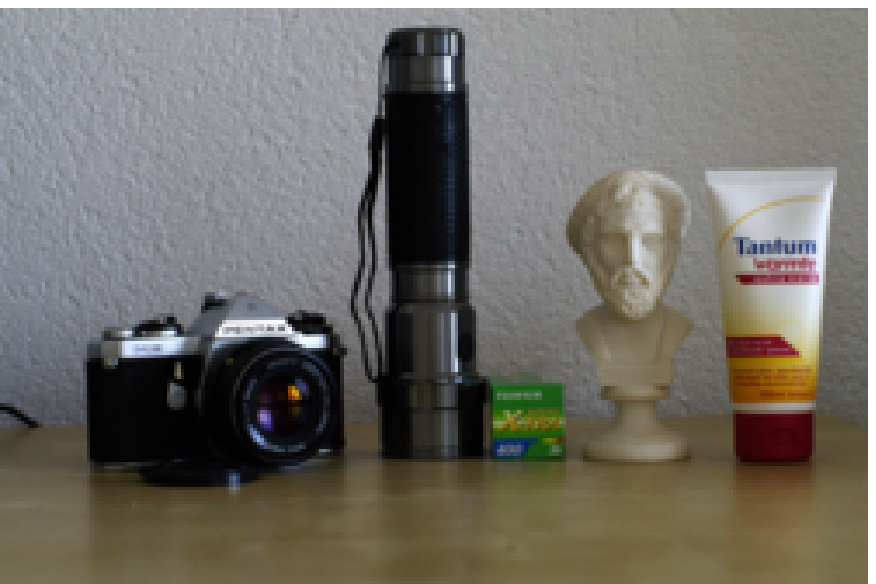}
  \hfill
  \includegraphics[width=\tomwidth\textwidth,height=\tomheight\textheight]{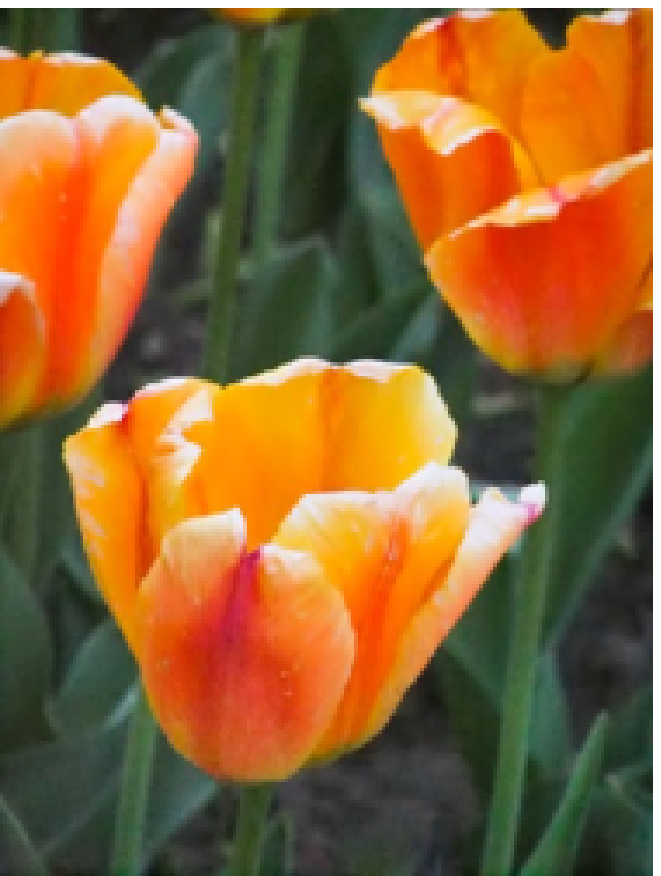}
  \hfill
  \includegraphics[width=\tomwidth\textwidth,height=\tomheight\textheight]{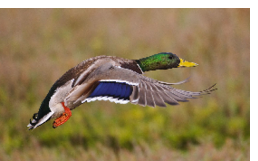}
  \hfill
  \includegraphics[width=\tomwidth\textwidth,height=\tomheight\textheight]{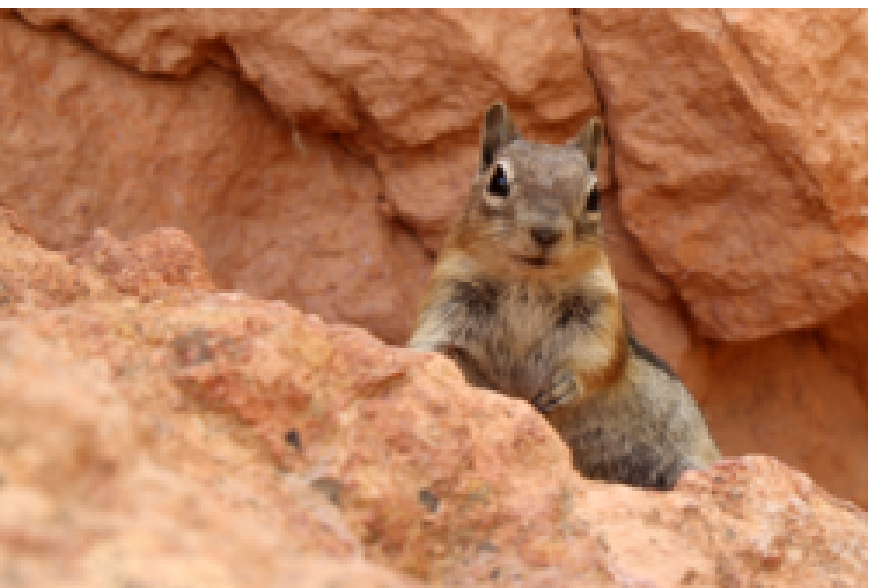}
  \hfill
  \includegraphics[width=\tomwidth\textwidth,height=\tomheight\textheight]{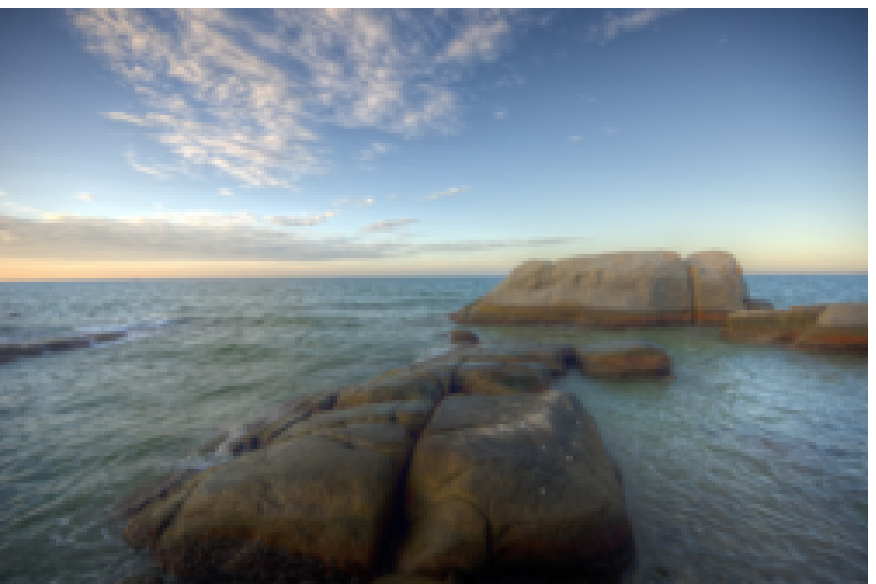}
  \hfill
  \includegraphics[width=\tomwidth\textwidth,height=\tomheight\textheight]{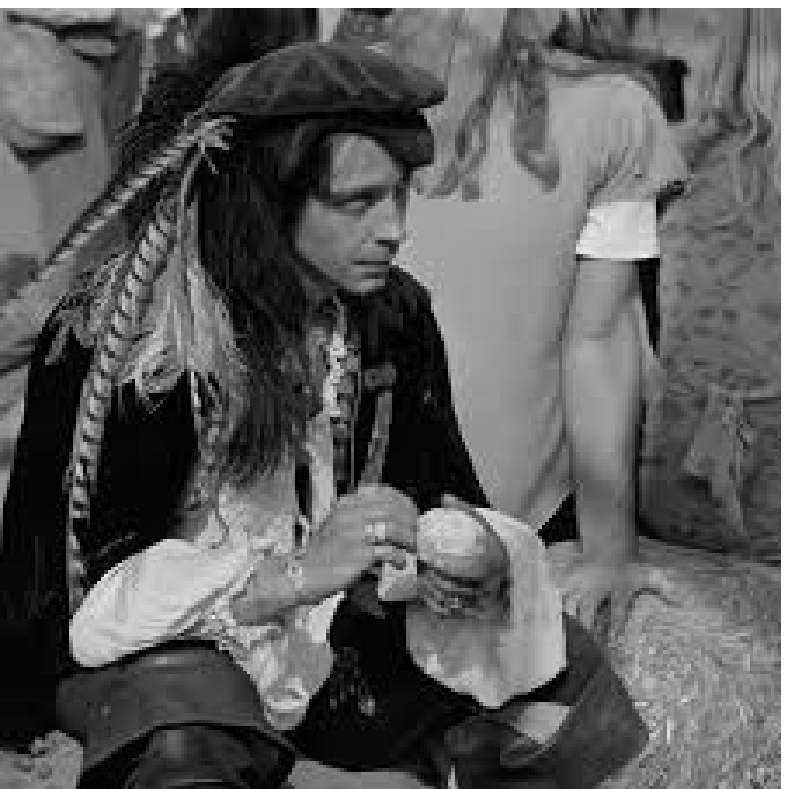}
  \hfill
  \includegraphics[width=\tomwidth\textwidth,height=\tomheight\textheight]{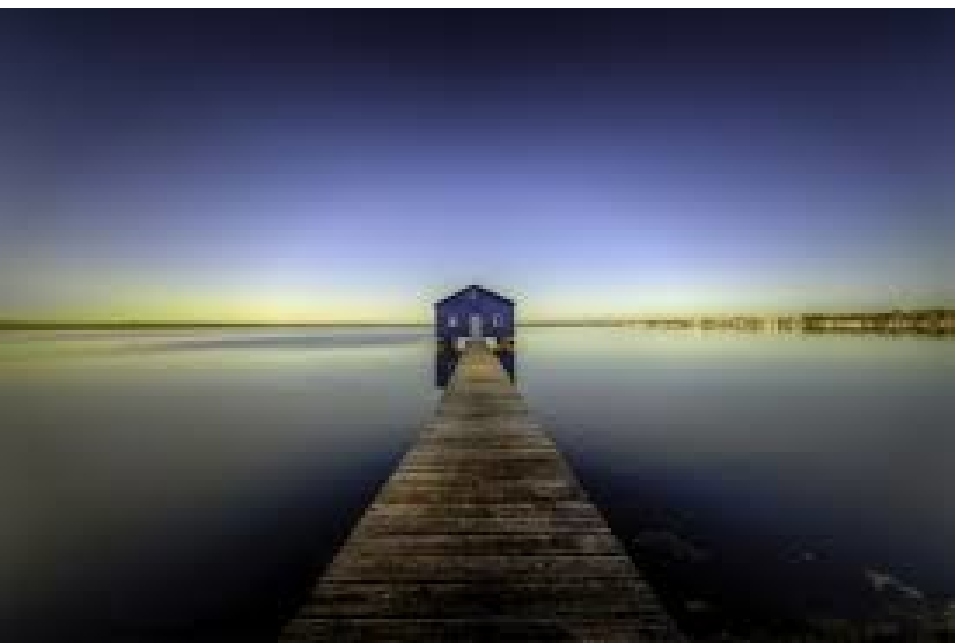}
  \hfill
  \includegraphics[width=\tomwidth\textwidth,height=\tomheight\textheight]{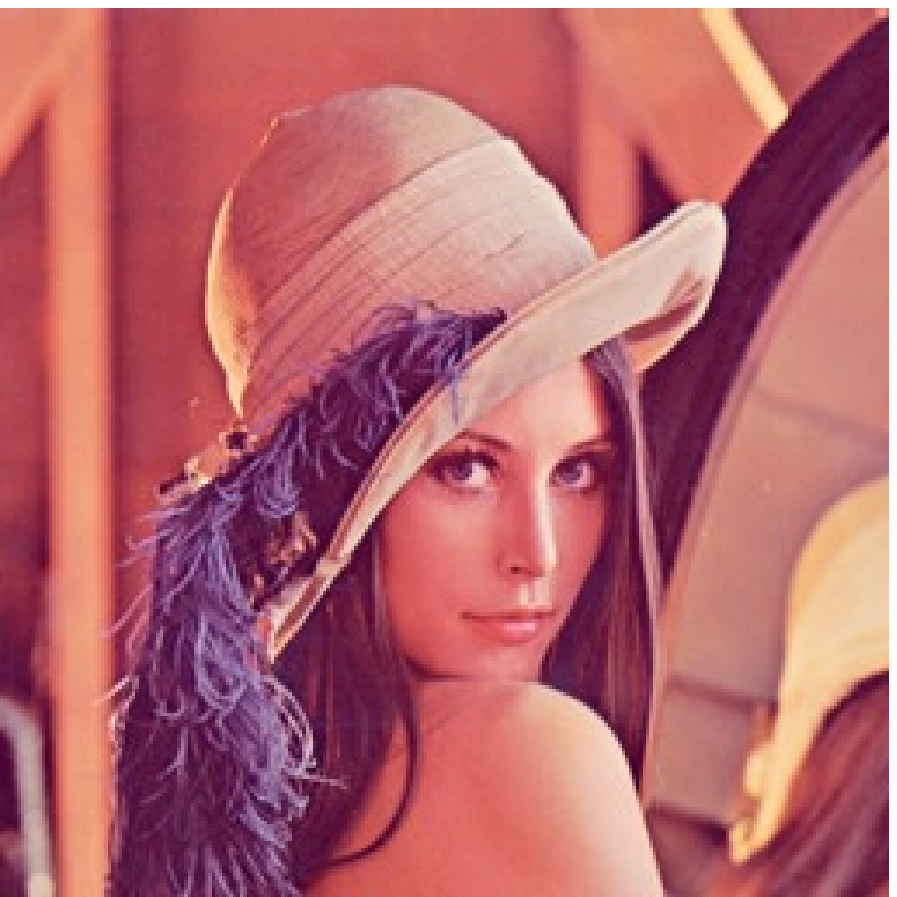}
  \hfill
  \includegraphics[width=\tomwidth\textwidth,height=\tomheight\textheight]{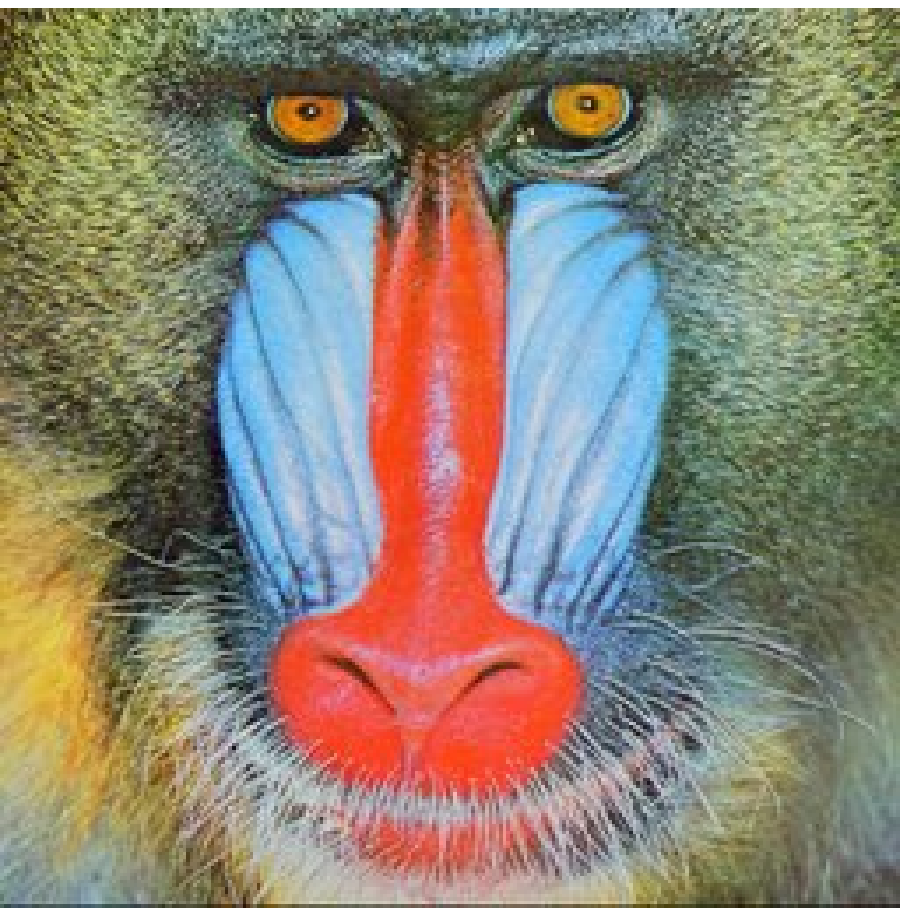}
  \caption{Set of 11 Test Images}
  \label{fig:images}
\end{figure}

We summarize our results in Table~\ref{table:keypoints}.
In the $32\times 32$ case, we saw that in ten out of the eleven images the equivalent keypoints were found (a few not having keypoints to be found).
Investigating the only image which the encrypted algorithm found a feature while the encrypted algorithm did not, we observed that the neighborhood of the keypoint there was another location that had a value very close to the keypoint's response value.
Figure~\ref{fig:intensity} graphically displays intensities of the neighborhood.
Clearly the point in the top left corner from the layer below is showing to be very close to the center point.
We believe this to be introduced by the rational conversion error, but it brings in the question of the stability of the unencrypted keypoint. 
As we moved into $64\times 64$ images, we were able to see more differences caused by the initial rational error.
If we sum up the total numbers compared to the original image, we see that we are accurately computing about $82.5\%$ of the original points.

  \begin{table}
    \centering
    \caption{Results of Keypoint Comparison}
    \begin{tabular}{|llccccccccccc|}
      \hline
      Size & Type & 1 & 2 & 3 & 4 & 5 & 6 & 7 & 8 & 9 & 10 & 11 \\
      \hline
      32x32 & Un. & 1 & 0 & 0 & 0 & 0 & 1 & 1 & 0 & 1 & 1 & 1 \\
      & Enc. & 1 & 0 & 1 & 0 & 0 & 1 & 1 & 0 & 1 & 1 & 1 \\
      \hline
      64x64 & Un. & 22 & 16 & 15 & 22 & 22 & 22 & 7 & 0 & 1 & 25 & 19 \\
       & Enc. & 19 & 14 & 11 & 19 & 19 & 22 & 8 & 0 & 2 & 20 & 17 \\
      \hline
    \end{tabular}
    \label{table:keypoints}
  \end{table}

\begin{figure}
  \centering

  \includegraphics[scale=0.20,angle=270]{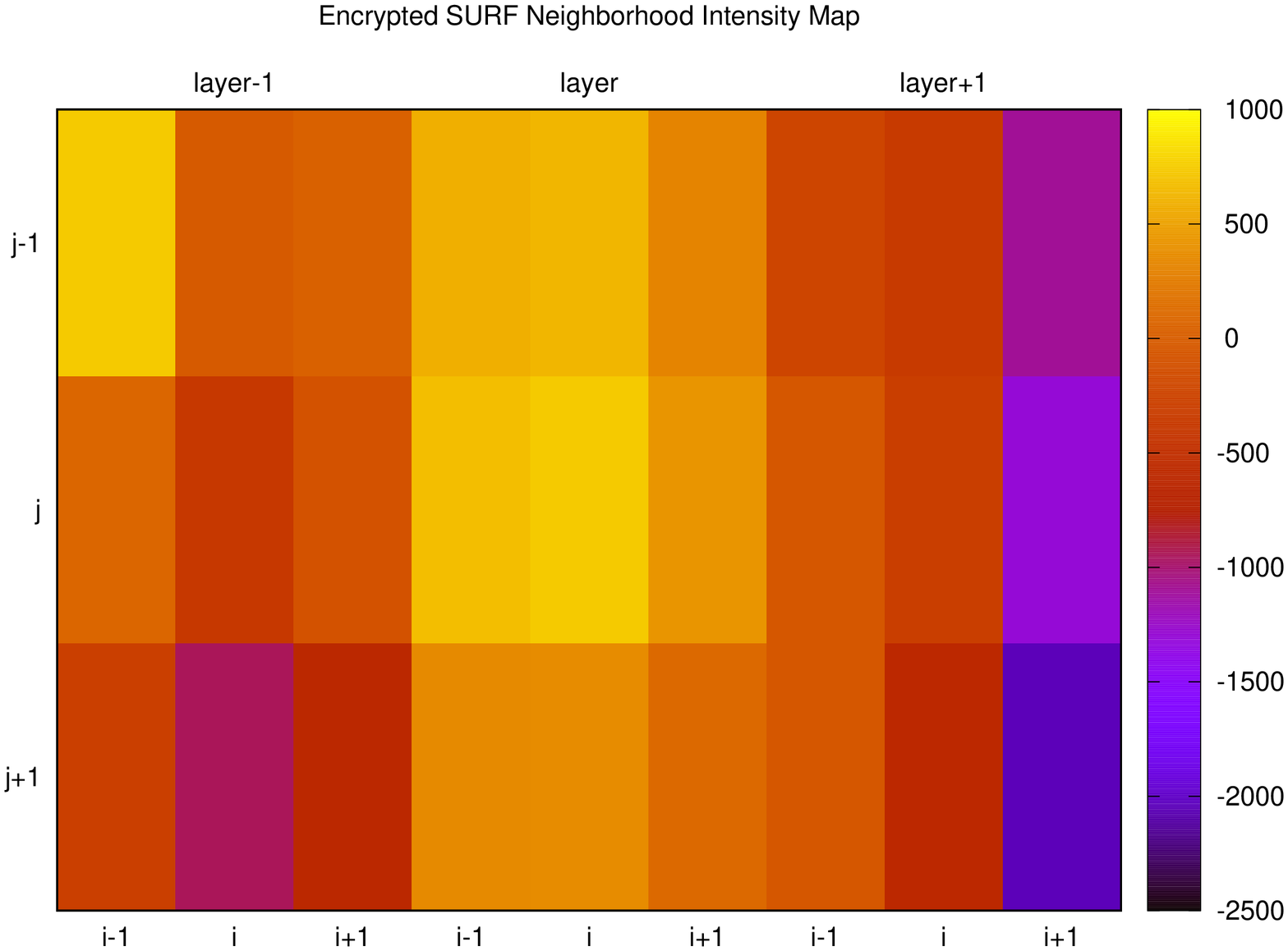}
  \includegraphics[scale=0.20,angle=270]{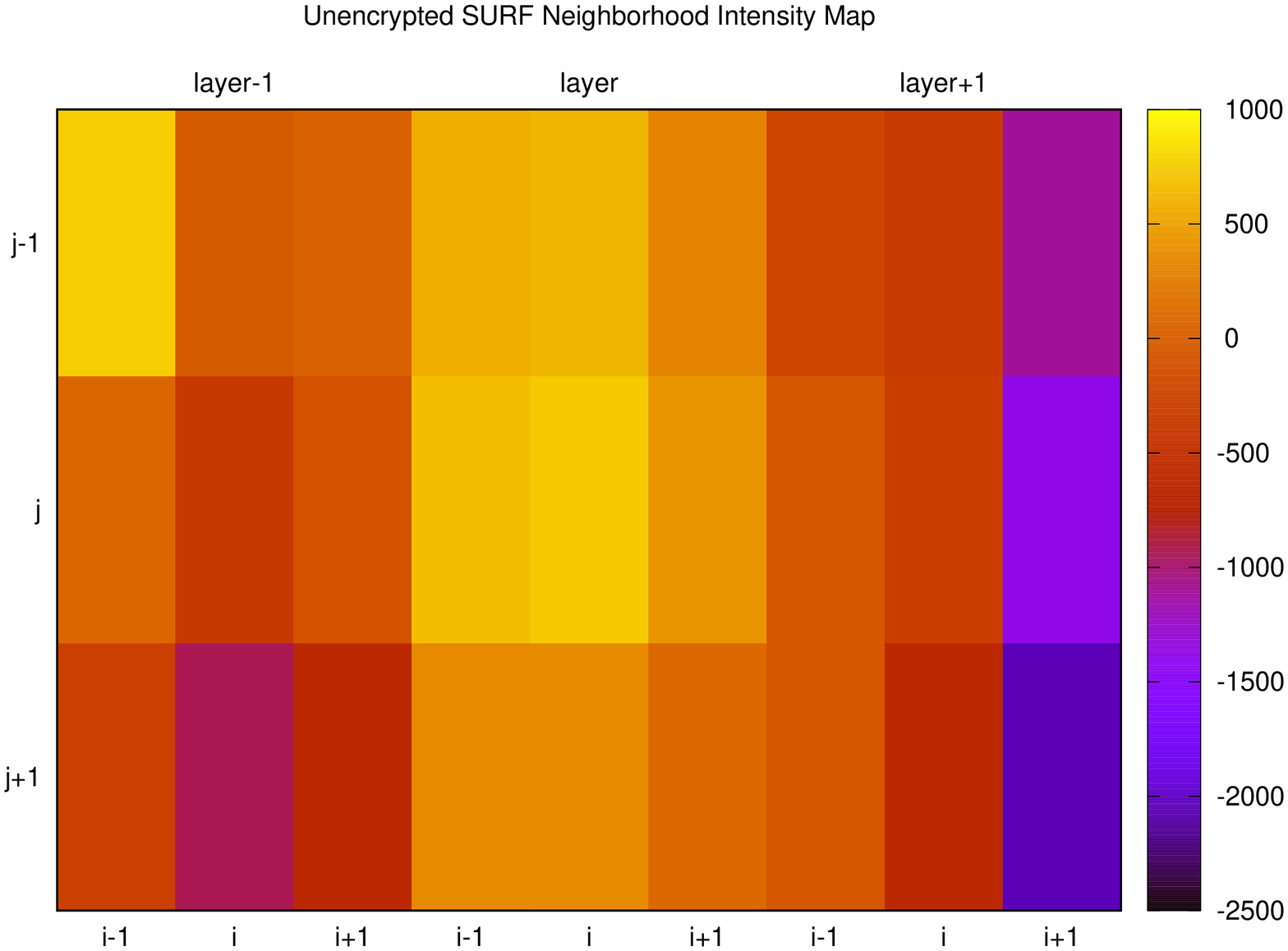}

  \caption{Intensity Map of the Encrypted and Unencrypted section of Image 3}
  \label{fig:intensity}
\end{figure}

\paragraph{Time/Space Complexity}

FHE is known to have poor time and space complexity.
A single ciphertext is represented by a $O(N^2)$ matrix ($N$ is defined by FHE scheme).
This size drives the time complexity of the FHE operations.
FHE addition is a matrix-matrix addition that is a $O(N^2)$ operation.
Multiplication is matrix-matrix multiplication which is a $O(N^3)$ process.
Since our scheme is mostly using additions and multiplications, we are starting with a $O(N^3)$ process.
One final note about subtraction is that a constant multiplication (-1) is performed, so a subtraction is a $O(N^3)$ as well. 

Starting with space complexity, we must remember we are working with an image.
This means we need space $O(m n N^2)$ to store a single image.
For the scale space pyramid, we know that each image can be at least $m \times n$.
However, when moving through the octaves the size decreases by powers of two.
We can bound the space complexity on an upper bound by $O(o\cdot l \cdot m \cdot n \cdot N^2)$, where we have used $o$ to indicate the number of octaves and $l$ to be the number of layers per octave.
We note that the constant multiplier to this bound is a fractional value. 

Looking at the time complexities, we have two processes to consider: the integral image and the scale space pyramid.
Generating of the integral image is a $O(mn)$ process that uses two additions and a subtraction for its calculations.
This will yield a $O(mnN^3)$ time complexity.
Moving on to the scale space pyramid.
We know from the space complexity that there are $O(o\cdot l \cdot m \cdot n)$ points to calculate.
Each point is computed for both the determinant and trace via Haar patterns.
Each of these is a constant number of combination of additions, subtractions, and multiplications.
This means the computations in time are $O(N^3)$.
For the scale-space pyramid, this means a time complexity of $O(o\cdot l \cdot m \cdot n \cdot N^3)$.

The space and time complexities need to be dealt with. 
Another important aspect of SURF is that it can be easily parallelized.
Each point of the scale-space pyramid can be computed independently once the integral image has been calculated.
Another parallelizable part of this framework is the matrix-matrix multiplication of the FHE scheme (we have successfully be able to use GPUs to significantly improve running time).
This enables an additional method to improve the scheme's overall running time.% because of the complexity.
As a concrete example, we can compute a $64\times 64$ image in under eight hours by using a 4-CPU Desktop computer with a GPU. 
An important takeaway of this section is that complexity of FHE demands only a necessary amount of computations.

\section{Conclusions}
\label{sec:conclusions}

Performing SURF in the encrypted domain provides a method for a user to offload their computations.
Moving this computation to a cloud computer enables other potential actions to be taken post the SURF process.
One problem we did not solve was the ability to identify the keypoints in the encrypted domain and be able to interpolate these points.
Other feature extraction algorithms could potentially be computed in FHE as a result of this as well.

\bibliographystyle{abbrv}
\bibliography{surfconferencepaper}

\end{document}